\documentclass[a4paper]{article}
\usepackage[margin=1.2in]{geometry}
\usepackage{amsmath}
\usepackage{algorithmic}
\usepackage{algorithm}
\usepackage{amsthm}
\usepackage{amsfonts}
\usepackage{comment}
\usepackage{graphicx}
\usepackage{hyperref}
\usepackage{color}
\usepackage{algorithm,algorithmic}

\newtheorem{thm}{Theorem}
\newtheorem{definition}{Definition}
\newtheorem{lemma}{Lemma}

\newtheorem{remark}{Remark}
\newtheorem{cor}[thm]{Corollary}

\usepackage[mathscr]{euscript}

\usepackage[mathscr]{euscript}

\usepackage{algorithm,algorithmic}

\usepackage{hyperref}
\hypersetup{
colorlinks=true,
citecolor=blue,
linkbordercolor={1  0 1}, 
citebordercolor={1 1 1} 
}

\def \R {\mathbb{R}}
\def \x {\mathbf{x}}
\def \E {\mathrm{E}}

\def \S {\mathcal{S}}

\def \sgn {\mbox{sign}}

\def \sgn {\mbox{sgn}}

\def \Rh {\widehat{R}}
\def \fh {\widehat{f}}

\def \X {\mathcal{X}}
\def \Y {\mathcal{Y}}
\def \F {\mathcal{F}}
\def \H {\mathcal{H}}

\def \E {\mathbb{E}}

\def \Ex {\mathscr{E}}

\def \F {\mathscr{F}}
\def \X {\mathscr{X}}
\def \Y {\mathscr{Y}}
\def \H {\mathscr{H}}

\def \f \mathscr{f}
\def \S {\mathscr{S}}
\def \Pxy {P_{\X \times \Y}}
\def \Rp {R_{\phi}}
\title{Binary Excess Risk for Smooth Convex Surrogates}

\author{Mehrdad Mahdavi \\
\small{Michigan State University} \\ 
\small{\texttt{mahdavim@cse.msu.edu}}
\and {Lijun Zhang} \\
\small{Michigan State University} \\ 
\small{\texttt{zhanglij@msu.edu}} 
\and Rong Jin\\  
\small{Michigan State University} \\
\small{\texttt{rongjin@cse.msu.edu}} 
}
\date{}
\begin{document} 
\maketitle

\begin{abstract}
In  statistical learning theory, convex surrogates of the  0-1 loss are highly preferred because of the computational  and theoretical virtues that convexity brings in.  This is  of more importance if we consider smooth   surrogates  as  witnessed by the fact that the smoothness  is  further beneficial both computationally- by attaining an {\it optimal}  convergence rate for optimization, and in a statistical sense- by providing an improved {\it optimistic} rate for generalization bound. In this paper we investigate the smoothness property from the viewpoint of statistical consistency and show  how it affects the binary excess risk. We show that in contrast to optimization   and generalization errors  that favor the choice of smooth surrogate loss, the smoothness of loss function may degrade the binary excess risk. Motivated by this negative result, we provide a unified analysis that integrates optimization error, generalization bound, and the error in translating convex excess risk into a binary excess risk when examining the impact of smoothness on the binary excess risk.  We  show that under favorable conditions appropriate choice of smooth convex loss will result in a binary excess risk that is better than $O(1/\sqrt{n})$. 
\end{abstract}

%
\section{Introduction}
In statistical binary classification, we are given an instance space $\X$, and a space $\Y =  \{-1,+1\}$ of labels  representing which of two classes an instance belongs to.  We assume that the product space $\X \times \Y$ is endowed with a probability measure $\Pxy$ that  is unknown to the learner,   however,  it is possible to sample  an arbitrary finite number of independent and identically distributed  observations $\S = \left((\x_1, y_1), (\x_2,y_2), \cdots, (\x_n,y_n) \right) \in \left(\X\times\Y\right)^n$ from $\Pxy$. The goal is to output a classifier $f: \X \mapsto \Y$ from a predefined  class of functions $\F$  that does well on unseen examples coming from the same distribution.  The most natural loss function to measure the performance of a classifier   is the misclassification error rate or the 0-1 (binary) loss  $\ell_{0-1}(f;(\x,y)) = \mathbb{I} [y f(\x) \leq 0]$, where  $\mathbb{I}[z]$ is the indicator function that outputs $1$ if $z$ is true and zero, otherwise. We are  interested, with the help of labeled examples  $\S$,  in finding a classifier $f \in \F$ that minimizes the {\it binary excess risk},
\[
R(f) - R^* = \Pr_{(\x,y)\sim\Pxy}\left\{yf(\x) \leq 0\right\} - \min\limits_{g} \Pr_{{(\x,y)\sim\Pxy}}\left\{yg(\x) \leq 0\right\}.
\]
Many studies of binary excess risk assume that the optimal classifier $f \in \F$ is learned by minimizing the
empirical binary risk, $\min_{f \in \F} \frac{1}{n}\sum_{i=1}^n \mathbb{I}[y_if(\x_i) \leq 0]$, an approach that is usually referred to as Empirical Risk Minimization (ERM)~\cite{vapnik1998statistical}. To understand the generalization performance of the classifier learned by ERM, it is important to have upper bounds on the excess risk of the empirical minimizer that hold with a high probability and that take into account complexity measures of classification functions. It is well known that, under certain conditions, direct empirical classification error minimization is consistent~\cite{vapnik1998statistical}  and achieves a fast convergence rate under low noise situations~\cite{mammen1999smooth}.

One shortcoming of the ERM based approaches is that they need to minimize  0-1 loss, leading to non-convex optimization problems that are potentially NP-hard~\footnote{We note that several works~\cite{kalai2008agnostically,kalai2009isotron} provide efficient algorithms for direct 0-1 empirical error minimization but under strong (unrealistic) assumptions on data distribution or label generation.}~\cite{arora1993hardness,hoffgen1995robust}.   A common practice  to circumvent this difficulty is to replace the indicator function $\mathbb{I} [\cdot \leq 0]$ with some convex loss  $\phi(\cdot)$ and find the optimal solution by minimizing the convex surrogate loss. Examples of such surrogate loss functions for 0-1 loss  include  logit loss $\phi_{\log}(f; (\x,y)) = \log (1+\exp(-y f(\x)))$ in logistic regression~\cite{friedman2000additive},  hinge loss $\phi_{\text{hing}}(f; (\x,y)) = \max(0, 1- y f(\x))$ in support vector machine (SVM)~\cite{cortes1995support} and exponential loss $\phi_{\text{exp}}(f; (\x,y)) = \exp(-y f(\x))$ in AdaBoost~\cite{freund1995desicion}.

It is known that under  mild conditions, the classifier learned by minimizing the empirical loss of convex surrogate is consistent to the Bayes classifier ~\cite{zhang2004statistical,lugosi2004bayes,jiang2004process,lin2004note,steinwart2005consistency,bartlett-2006-convexity}.  For instance, it was shown in~\cite{bartlett-2006-convexity} that the necessary and sufficient condition for a convex loss $\phi(\cdot)$ to be consistent with the binary loss is that $\phi(\cdot)$ is  differentiable at origin and $\phi'(0) < 0$. It was further established in the same work that the binary excessive risk can be upper bound by the convex excess risk through a $\psi$-transform that depends on the surrogate convex loss $\phi(\cdot)$.

Since the choice of convex surrogates could significantly affect the binary excess risk, in this work, we will investigate the impact of the smoothness of a convex loss function on the binary excess risk. This is motivated by the recent results that show the advantages of using smooth convex surrogates in reducing the optimization complexity and the generalization error bound. More specifically, ~\cite{nesterov2004introductory,tseng:2009:accelerated} show that a faster convergence rate (i.e. $O(1/T^2)$) can be achieved by first order methods when the objective function to be optimized is convex and smooth; in~\cite{srebro:2010:smoothness}, the authors show that a smooth convex loss will lead to a better optimistic generalization error bound. Given the positive news of using smooth convex surrogates, an open research question is how the smoothness of a convex surrogate will affect the binary excess risk. The answer to this question, as will be revealed later, is negative: the smoother the convex loss, the poorer approximation will be for the binary excess risk. Thus, the second contribution of this work is to integrate  these results for smooth convex losses, and examine the overall effect of replacing 0-1 loss with a smooth convex loss when taking into account three sources of errors, i.e. the optimization error, the generalization error, and the error in translating the convex excess risk into the binary risk. As we will show, under favorable conditions, appropriate choice of smooth convex loss will result a binary excess risk better than $O(1/\sqrt{n})$.

\vspace{-0.1cm}
\paragraph{Outline}{The reminder of this paper is organized as follows. In Section~\ref{sec:setting} we set up notation and describe the setting. Section~\ref{sec:calibration} briefly discusses the classification-calibrated convex surrogate losses  on which our analysis relies. We derive the $\psi$-transform for smoothed hinge loss and elaborate  its binary excess risk in Section~\ref{sec:smoothed-hinge}. Section~\ref{sec:analysis} provides a unified analysis of three types of errors and derives conditions  in terms of smoothness to obtain better rates for the binary excess risk. Section~\ref{sec:conclusion} concludes the paper and the appendix contains the omitted proofs. }

%
\section{Problem Setting} \label{sec:setting}
Let $\S = \left((\x_1, y_1), (\x_2, y_2), \cdots, (\x_n, y_n) \right)$ be a set of i.i.d. samples drawn from an unknown distribution $\Pxy$ over $\X\times \{-1, +1\}$, where $\x_i \in \X \subseteq \R^d$ is an instance and $y_i \in \{-1, +1\}$ is the binary class assignment for $\x_i$. Let $\kappa(\cdot, \cdot)$ be an universal kernel  and let $\H$ be the Reproducing Kernel Hilbert Space (RKHS) endowed with kernel $\kappa(\cdot, \cdot)$. According to~\cite{zhou2003capacity}, $\H$ is a rich function space whose closure includes all the smooth functions.  We consider  predictors from $\H$ with bounded norm to form the measurable function class $\F = \{f \in \H: \|f\|_{\H} \leq B\}$. Given a  function  $f:\X \mapsto \R$, define the risk of $f$  as $R(f) = \E_{(\x,y)\sim\Pxy}\left[ \mathbb{I} [yf(\x) \leq 0]\right] = \Pr\left\{yf(\x) \leq 0 \right\}$. Let $f^*$ be the optimal classifier that attains the minimum risk, i.e. $f = \arg\min_{f} \Pr\left\{yf(\x) \leq 0 \right\}$. We assume $f^* \in \H$ with $\|f^*\|_{\H} \leq B$. This boundedness condition  is satisfied for any RKHS with a bounded kernel (i.e. $\sup_{\x \in \X} \kappa(\x, \x) \leq B$). Henceforth, let $R^*$ stand for the minimum achievable risk by the optimal classifier $f^*$, i.e., $R^* = R(f^*)$. Define the \textit{binary excess risk} for a prediction function $f \in \F$ as
$$\Ex(f) = R(f) - R^*.$$
Our goal is to efficiently learn a prediction function $f \in \F$ from the training examples  in $\S$ that minimizes the binary excess risk $\Ex(f)$. As already mentioned, for computational convenience, we work with a convex loss function rather than the 0-1 loss function.  Given a convex surrogate loss function $\phi: \R \mapsto \R_{+}$ (e.g., hinge loss, exponential loss, or logistic loss) we define the risk with respect to the convex loss $\phi$ (convex risk or $\phi$-risk) as $$\Rp(f) = \E_{(\x,y)\sim\Pxy}[\phi(yf(\x))].$$ Similarly we define the \textit{optimal} $\phi$-risk  as $\Rp^* = \inf_{f \in \F} \E_{(\x,y)\sim \Pxy}[\phi(yf(\x))]$. The \textit{excess $\phi$-risk} or \textit{convex excess  risk} of a classifier $f \in \F$ with respect to the convex surrogate loss $\phi(\cdot)$ is defined as
$$\Ex_{\phi}(f) = R_{\phi}(f) - R_{\phi}^*.$$
An important line of research in statistical learning theory focused on relating the convex excess risk $\Ex_{\phi}(f)$ to the binary excess risk $\Ex(f)$ that will be elaborated in next section.
%
\section{Classification Calibration and Surrogate Risk Bounds}\label{sec:calibration}
Although it is computationally convenient to minimize the empirical risk based on a convex surrogate, the ultimate goal of any classification method is to find a function $f$ that minimizes the binary loss. Therefore, it is crucial to investigate the conditions which guarantee that if the $\phi$-risk of $f$ gets close to the optimal $R_{\phi}^*$, the binary risk of $f$ will also approach the optimal binary risk $R^*$.  This question has been an active trend in statistical learning theory over the last decade where the necessary and sufficient conditions have been established for relating the binary excess risk to a convex excess risk ~\cite{zhang2004statistical,lugosi2004bayes,jiang2004process,lin2004note,steinwart2005consistency,bartlett-2006-convexity}.

In this paper we follow the strategy introduced in~\cite{bartlett-2006-convexity} in order
to relate the binary excess risk to the excess $\phi$-risk. Their methodology, through  the notion of    \textit{classification calibration}, allows us to find quantitative
relationship between the excess risk associated with $\phi$ and the excess risk associated with 0-1 loss.  It is  established in~\cite{bartlett-2006-convexity} that the binary excessive risk can be bounded by the convex excess risk, based on the convex loss function $\phi$, through a $\psi$-transform.

\begin{definition} Given a loss function $\phi: \R \mapsto [0, \infty)$, define the function  $\psi: [0,1] \mapsto [0, \infty)$ by
$$\tilde{\psi} (z) = H^-\left(\frac{1+z}{2}\right) - H\left(\frac{1+z}{2}\right)$$
where
\[
H^-(\eta)=\inf_{\alpha:\alpha(2\eta-1)\leq 0}\left( \eta \phi(\alpha)+(1-\eta) \phi(-\alpha)\right)\textrm{ and } H(\eta)= \inf_{\alpha\in\R} \left( \eta \phi(\alpha)+(1-\eta) \phi(-\alpha)\right).
\]
The transform function $\psi: [0,1]\mapsto [0,\infty)$ is defined to be the convex closure of $\tilde{\psi}$.
\end{definition}
The following theorem from~\cite[Theorem 1]{bartlett-2006-convexity} shows that the binary excess risk can be bounded by the convex excess risk using transform function $\psi: [0,1] \mapsto [0, \infty)$ that depends on the surrogate convex loss function.
\begin{thm} \label{thm:classification-caliberated} For any non-negative loss function $\phi(\cdot)$, any measurable function $f \in \F$, and any probability distribution $\Pxy$ on $\X\times\Y$, there is  a nondecreasing function $\psi: [0,1] \mapsto [0, \infty)$ that
\begin{eqnarray}\label{eqn:bound-phi}
\psi(R(f) - R^*) \leq R_{\phi}(f) - R_{\phi}^*
\end{eqnarray}
holds. Here the minimization is taken over all measurable functions.
\end{thm}
\begin{definition} A convex loss $\phi$ is classification-calibrated if, for any $\eta \neq 1/2$,
\[
H^-(\eta) > H(\eta).
\]
\end{definition}
This condition is essentially an extension of~\cite[Theorem 2.1]{zhang2004statistical} and can be viewed as a form of Fisher consistency that is appropriate for classification.

It has been shown  in~\cite{bartlett-2006-convexity} that the necessary and sufficient condition for a convex loss $\phi(z)$ to be classification-calibrated  is if it is differentiable at the origin  and $\phi'(0) < 0$. In particular, for a certain convex function $\phi(\cdot)$, the $\psi$-transform can be computed by
\begin{eqnarray*}
\psi(z) = \inf\limits_{\alpha z \leq 0}\left(\frac{1 + z}{2}\phi(\alpha) + \frac{1 - z}{2}\phi(-\alpha)\right) - \inf\limits_{\alpha \in \R}\left(\frac{1 + z}{2}\phi(\alpha) + \frac{1 - z}{2}\phi(-\alpha)\right),
\end{eqnarray*}
that can be further simplified as $\psi(z) = \phi(0)-H\left(\frac{1+z}{2}\right)$ when $\phi$ is classification-calibrated.  Examples of $\psi$-transform for the convex surrogate functions of known practical algorithms mentioned before are as follows: (i) for hinge loss $\phi(\alpha) = \max(0, 1-\alpha)$ , $\psi(z) = |z|$, (ii) for exponential loss $\phi(\alpha)=e^{-\alpha}$, $\psi(z) = 1 - \sqrt{1 - z^2} \geq z^2/2$, and (iii) for truncated quadratic loss $\phi(\alpha) = [\max(0, 1 - \alpha)]^2$, $\phi(z) = z^2$.

\begin{remark} We note that the inequality in (\ref{eqn:bound-phi}) provides insufficient guidance on choosing appropriate loss function. First, it does not measure explicitly how the choice of the convex surrogate $\phi(\cdot)$ affects the excess risk $R_{\phi}(f) - R_{\phi}^*$. Second, it does not take into account the impact of loss function on optimization efficiency, an important issue for practitioners when dealing with big data. It is thus unclear, from Theorem~\ref{thm:classification-caliberated}, how to choose an appropriate loss function that could result in a small generalization error for the binary loss when the computational time is limited. In this paper, we address these limitations by examining a family of convex losses that are constructed by smoothing the hinge loss function using different smoothing parameters. We study the binary excessive risk of the learned classification function by taking into account errors in optimization, generalization, and translation of convex excess risk into binary excess risk.
\end{remark}

%
\section{Binary Excess Risk for Smoothed Hinge Loss}\label{sec:smoothed-hinge}

As stated before, to efficiently learn a prediction function  $f \in \F$, we will replace the binary loss with a smooth convex loss. Since hinge loss is one of the most popular loss functions used in machine learning and is the loss of choice for classification problems in terms of the margin error~\cite{ben2012minimizing}, in this work, we will focus on the smoothed version of the hinge loss. Another advantage of using the hinge loss is that its $\psi$-transform is a linear function. Compared with the $\psi$-transforms of other popular convex loss functions (e.g. exponential loss and truncated square loss) that are mostly quadratic, using the hinge loss as convex surrogate will lead to a tighter bound for the binary excess risk.

The smoothed hinge loss considered in this paper is defined as
\begin{eqnarray}\label{eqn:smooth-hing}
\phi(z; \gamma) = \max\limits_{\alpha  \in [0, 1]} \alpha (1 - z) + \frac{1}{\gamma}\mathscr{R}(\alpha),
\end{eqnarray}
where $\mathscr{R}(\alpha) = -\alpha\log\alpha - (1 - \alpha)\log(1 - \alpha)$ and $\gamma > 0$ is the smoothing parameter.  It is straightforward  to verify that the loss function in~(\ref{eqn:smooth-hing}) can be simplified as
\[
\phi(z; \gamma) = \frac{1}{\gamma}\log(1 + \exp(\gamma(1 - z))).
\]
It is not immediately clear from Theorem~\ref{thm:classification-caliberated} how the relationship between smooth convex excess risk $\Ex_{\phi}(\cdot)$ and binary excess risk is affected by the smoothness parameter $\gamma$.
In addition, as discussed  in~\cite{bartlett-2006-convexity}, whereas conditions such as convexity and smoothness have natural relationship to optimization and generalization, it is not 
immediately obvious how properties such as convexity and smoothness of convex surrogate  relates to  statistical consequences.  In what follows,  we show that, indeed  smoothness of loss function has a \textit{negative} statistical consequence and can degrade the binary excess risk.

%
\subsection{$\psi$-Transform for smoothed hinge loss}
The first step in our analysis is to derive the $\psi$-transform for the loss function defined in~(\ref{eqn:smooth-hing}) as stated in the  following theorem.

\begin{thm} \label{thm:psi}  The  $\psi$-transform of smoothed hinge loss with smoothing parameter $\gamma$ is given by
\[
\psi(\eta; \gamma) = -\frac{1 + \eta}{2\gamma}\log\left(\frac{1}{1 + e^{\gamma}}\left[1 + e^{\gamma} \frac{C_1}{1 + \eta}\right]\right) -\frac{1 - \eta}{2\gamma}\log\left(\frac{1}{1 + e^{\gamma}}\left[1 + e^{\gamma} \frac{C_2}{1 - \eta}\right]\right)
\]
where $C_1$ and $C_2$ are defined as
$C_1 = -\eta e^{\gamma} + \sqrt{\eta^2e^{2\gamma} + 1 - \eta^2}$ and $C_2 = \eta e^{\gamma} + \sqrt{\eta^2e^{2\gamma} + 1 - \eta^2}$.
\end{thm}
The $\psi$-transform given in Theorem~\ref{thm:psi} is too complicated to be useful. The theorem below provides a simpler bound for the $\psi$-transform in terms of the smoothness parameter $\gamma$.
\begin{thm} \label{thm:psi-simple}
For $\eta \in (-1, 1)$, we have $$\psi(\eta; \gamma) \geq |\eta| - \frac{1}{\gamma}\log\frac{1}{|\eta|}.$$
\end{thm}
\begin{remark}The bound obtained in Theorem~\ref{thm:psi-simple} demonstrates that when $\gamma$ approaches to infinity, the $\psi$-transform for smoothed hinge loss $\phi(\eta; \gamma)$ becomes $|\eta|$.  According to~\cite{bartlett-2006-convexity}, the $\psi$-transform for the hinge loss is $\psi(\eta) = |\eta|$. Therefore, this result is consistent with the $\psi$-transform for smoothed  hinge loss, which is the limit of $\phi(z;\gamma)$ as $\gamma$ approaches infinity.
\end{remark}

\subsection{A bound on binary excess risk $\Ex(f)$ based on smooth convex excess risk $\Ex_{\phi}(f)$}

Based on the  transform function $\psi(\cdot; \gamma)$  that is  computed for smoothed hinge loss with smoothing parameter $\gamma$, we are now in the position to  bound its corresponding  binary excess risk $\Ex(f)$.  Our main result in this section is the following theorem  that shows  how binary excess risk can be bounded by the excess $\phi$-risk for smoothed hinge loss.
\begin{thm} \label{thm:bound-binary-excess-risk}
Consider any measurable function $f \in \F$ and the smoothed hinge loss $\phi(\cdot)$ with  parameter $\gamma$ defined in~(\ref{eqn:smooth-hing}). Then, binary excess risk $\Ex(f)$ can be bounded by the smooth convex excess risk $\Ex_{\phi}(f)$ as
\[
\Ex(f) \leq \Ex_{\phi}(f) + \frac{\Ex_{\phi}(f)}{1 + \gamma \Ex_{\phi}(f)} \log\frac{1}{\Ex_{\phi}(f)}.
\]
\end{thm}
\begin{proof}
Using the result from Theorem~\ref{thm:classification-caliberated}, we have $\Ex_{\phi}(f) \geq \psi(\Ex(f); \gamma)$ and therefore an immediate result from the $\psi$-transform for smoothed hinge loss that is obtained in Theorem~\ref{thm:psi-simple} indicates
\[
\Ex(f) + \frac{1}{\gamma}\log\Ex(f) \leq \Ex_{\phi}(f).
\]
Define $\Delta = \Ex(f) - \Ex_{\phi}(f)$. We have
\[
\Delta + \frac{1}{\gamma}\log(\Delta + \Ex_{\phi}(f)) = \Delta + \frac{1}{\gamma}\log\Ex_{\phi}(f) + \frac{1}{\gamma}\log\left(1 + \frac{\Delta}{\Ex_{\phi}(f)}\right) \leq 0.
\]
Based on the  $\log (1 + x) \leq x$ inequality, the sufficient condition for the above inequality to hold is to have
\[
\Delta + \frac{\Delta}{\gamma\Ex_{\phi}(f)} \leq \frac{1}{\gamma}\log\frac{1}{\Ex_{\phi}(f)}
\]
and therefore
\[
\Delta \leq \frac{\gamma^{-1}}{1 + (\gamma\Ex_{\phi}(f))^{-1}}\log\frac{1}{\Ex_{\phi}(f)} = \frac{\Ex_{\phi}(f)}{1 + \gamma\Ex_{\phi}(f)}\log\frac{1}{\Ex_{\phi}(f)}.
\]
The final bound is obtained by substituting $\Ex(f) - \Ex_{\phi}(f)$ for $\Delta$ in the left hand side of above inequality.
\end{proof}
As indicated by Theorem~\ref{thm:bound-binary-excess-risk}, the smaller the smoothing parameter $\gamma$, the poorer the approximation is in bounding the binary excess $\Ex(f)$ with smooth convex excess risk $\Ex_{\phi}(f)$. On the other hand, the smoothness of loss function has been proven to be beneficial in terms of optimization error and generalization bound. The mixture of negative and positive results  for using smooth convex surrogates motivates us to develop an integrated bound for binary excess risk that takes into account all types of errors. One of the main contributions of this work is to show that under favorable conditions, with appropriate choice of smoothing parameter, the smoothed hinge loss will result in a bound for the binary excess risk better than $O(1/\sqrt{n})$.
%
\section{A Unified Analysis of Binary Excess Risk for Smooth Surrogates}\label{sec:analysis}
Using the smoothed hinge loss, we define the convex loss for a prediction function $f \in \F$ as $R_{\phi}(f) = \E[\phi(yf(\x); \gamma)]$. Let $f_{\gamma}^*$ be the optimal classifier that minimizes $R_{\phi}(f)$. Similar to the case of binary loss, we assume $f_{\gamma}^* \in \H$ with $\|f_{\gamma}^*\| \leq B$. The {\it smooth convex excess risk} for a given prediction function $f \in \F$ is then given by $\Ex_{\phi} (f) = R_{\phi}(f) - R_{\phi}(f_{\gamma}^*)$.
Given the smooth convex loss $\phi(z;\gamma)$ in~(\ref{eqn:smooth-hing}), we find the optimal classifier by minimizing the empirical convex loss, i.e. $\min_{f \in \H, \|f\|_{\H} \leq B} \Rh_{\phi}(f)$, where the empirical convex loss $\Rh_{\phi}(f)$ is given by
\begin{eqnarray}\label{eqn:tseng-alg-1}
\Rh_{\phi}(f) = \frac{1}{n}\sum_{i=1}^n \phi(y_if(\x_i); \gamma).
\end{eqnarray}
Let $\fh$ be the solution learned from solving the empirical convex loss over training examples. There are three sources of errors that  affect bounding the binary excess risk $\Ex(\fh)$. First, since $\fh$ is obtained by numerically solving an optimization problem, the error in estimating the optimal solution, which we refer to as optimization error~\footnote{We note that in literature the error in estimating the optimal solution for empirical minimization is usually referred to as {\it estimation error}. We emphasize it as optimization error because different convex surrogates could lead to very different iteration complexities and consequentially different optimization efficiency.}, will affect $\Ex(\fh)$. Additionally, since the binary excess risk can be bounded by a nonlinear transform of the convex excess risk, both the bound for $\Ex_{\phi}(\fh)$ and the error in approximating $\Ex(\fh)$ with $\Ex_{\phi}(\fh)$ will affect the final estimation of $\Ex(\fh)$.  We aim at investigating how the smoothing parameter $\gamma$ affect all these three types of errors.  As it is investigated in~Theorem~\ref{thm:bound-binary-excess-risk},  a smaller smoothing parameter $\gamma$ will result  in a poorer approximation of $\Ex(\fh)$. On the other hand,  a smaller smoothing parameter $\gamma$ will result in a smaller estimation error and a smaller bound for $\Ex_{\phi}(\fh)$. Based on the understanding of how  smoothing parameter $\gamma$ affects the three errors, we identify the choice of $\gamma$ that results in the best tradeoff between all three error and consequentially a binary excess risk $\Ex(\fh)$ better than $O(1/\sqrt{n})$.

To investigate how the smoothing parameter $\gamma$ affects the binary excess risk $\Ex(\fh)$,  we intend to unify three types of errors.  The analysis is comprised of two components, i.e. bounding the binary excess risk $\Ex(f)$ by a smooth convex excess risk $\Ex_{\phi}(f)$ that has been established in Theorem~\ref{thm:bound-binary-excess-risk} and bounding $\Ex_{\phi}(f)$ for a solution $f$ that is suboptimal in minimizing the empirical convex loss $\Rh_{\phi}(f)$ that is the focus of this section.
%
\subsection{Bounding smooth  excess convex risk $\Ex_{\phi}(f)$}
We now turn to bounding the excess $\phi$-risk $\Ex_{\phi}(f)$ for the smoothed hinge loss.  To bound $\Ex_{\phi}(f)$ we need to consider two types of errors: optimization error due to the approximate optimization of the empirical $\phi$-risk, and the generalization error bound for the empirical risk minimizer.     After obtaining these two errors for smooth convex surrogates, we provide a unified bound on the excess $\phi$-risk $\Ex_{\phi}(f)$ of empirical convex risk minimizer in terms of $n$.

We begin  by bounding  the error arising from solving the optimization problem numerically. One nice property of smoothed hinge loss function is that both its first order and second order derivatives are bounded, i.e.
\[
|\phi'(z;\gamma)| = \left|\frac{\exp(\gamma(1 - z)}{1 + \exp(\gamma(1 - z))}\right| \leq 1, \quad \phi''(z;\gamma) = \gamma \frac{\exp(\gamma(1 - z))}{(1 + \exp(\gamma(1 - z)))^2} \leq \frac{\gamma}{4}.
\]
Due to  the smoothness of $\phi(z; \gamma)$, we can apply the accelerated optimization algorithm~\cite{nesterov2004introductory,tseng:2009:accelerated} to achieve an $O(1/k^2)$ convergence rate for the optimization, where $k$ is the number of iterations the optimization algorithm proceeds. More specifically, we will apply Algorithm 1 from~\cite{tseng:2009:accelerated} to solve the  numerical optimization problem in~(\ref{eqn:tseng-alg-1}) over the convex domain $\F = \{f \in \H: \|f\|_{\H} \leq B\}$ which results in the following  updating rules at  $s$th iteration:
\begin{eqnarray}\label{eqn:acc-gd}
\begin{aligned}
g_s &= (1-\theta_s)f_s + \theta_sh_s \\
h_{s+1} &= \arg \min_{h \in \F} \left ( \langle \nabla \Rh_{\phi}(g_s), h - g_s \rangle + \frac{\theta_s}{2}\|h - h_s\|_{\H} \right) \\
f_{s+1} &= (1-\theta_s)f_s + \theta_s h_{s+1}.
\end{aligned}
\end{eqnarray}
The following theorem that  follows immediately from~\cite[Corollary 1]{tseng:2009:accelerated} and the fact $\phi''(z;\gamma) \leq \gamma/4$, bounds the optimization error for the optimization problem after  $k$ iterations.
\begin{lemma} \label{lemma:opt}
Let $\fh = f_{k+1}$ be the solution obtained by running accelerated gradient descent  method (i.e., updating rules in~(\ref{eqn:acc-gd})) to solve the optimization problem in~(\ref{eqn:tseng-alg-1}) after $k$ iterations with $\theta_0 = 1$ and $\theta_k = 2/(k+2)$ for $k \geq 1$. We have
\[
\Rh_{\phi}(\fh) \leq \min\limits_{\|f\|_{\H} \leq B} \Rh_{\phi}(f) + \frac{\gamma B^2}{(k+2)^2}.
\]
\end{lemma}
We now turn to understanding the generalization error for the smooth convex loss.   There are many theoretical results giving upper bounds of the generalization error. However, a recent result ~\cite{srebro:2010:smoothness} has showed  that it is possible to obtain optimistic rates  for generalization bound of smooth convex loss (in the sense that smooth losses yield better generalization bounds when the problem is easier), which are more appealing than the generalization of simple Lipschitz continuous losses. The following  theorem from~\cite[Theorem 1]{srebro:2010:smoothness} bounds the generalization error for any solution $f \in \F$  when the learning has been performed by a smooth convex surrogate $\phi(\cdot)$. 
\begin{lemma} \label{lemma:generalization} With a probability $1 - \delta$, for any $\|f\|_{\H} \leq B$, we have
\begin{eqnarray*}
R_{\phi}(f) - \Rh_{\phi}(f) & \leq & K_1\left(\frac{(B + \gamma B^2)t}{n} + \sqrt{\Rh_{\phi}(f)\frac{(B + \gamma B^2)t}{n}}\right)\\
R_{\phi}(f) - \Rh_{\phi}(f) & \leq & K_2\left(\frac{(B + \gamma B^2)t}{n} + \sqrt{R_{\phi}(f)\frac{(B + \gamma B^2)t}{n}}\right).
\end{eqnarray*}
where $t = \log(1/\delta) + \log^3n$ and $K_1$ and $K_2$ are  universal constants.
\end{lemma}
The bound stated in this lemma is \textit{optimistic} in the sense that it reduces to $\tilde{O}(1/\sqrt{n})$  when the problem is difficult and be better when the problem is easier, approaching $\tilde{O}(1/n)$ for  linearly separable data, i.e., $R_{\phi}^* = 0$ in the second inequality. These two  lemmas essentially enable us to transform a bound on the optimization error and generalization bound  into a bound on the convex excess risk. In particular, by combining Lemma~\ref{lemma:opt} with Lemma~\ref{lemma:generalization}, we have the following theorem that bounds the smooth convex excess risk $\Ex_{\phi}(\fh) = R_{\phi}(\fh) - R_{\phi}(f_{\lambda}^*)$ for the empirical convex risk minimizer.
\begin{thm}\label{thm:combined-1}
Let $\fh$ be the solution output from updating rules in~(\ref{eqn:acc-gd})  after $k$ iterations. Then, with a probability at least $1 - \delta$, we have
\begin{eqnarray*}
\Ex_{\phi}(\fh)  \leq \frac{\gamma B^2}{(k+2)^2} + K\left(\frac{(B + \gamma B^2)t}{n} + \sqrt{R_{\phi}^*\frac{(B + \gamma B^2)t}{n}} + \sqrt{\frac{\gamma B^2(B + \gamma B^2)t}{(k+2)^2 n}}\right)
\end{eqnarray*}
where $K$ is a universal constant, $t = \log(1/\delta) + \log^3n$,  and $R_{\phi}^* = \min_{\|f\|_{\H} \leq B} R_{\phi}(f)$.
\end{thm}
Since our overall interest is to understand how the smoothing parameter $\gamma$ affects the convergence rate of excess risk in terms of $n$, the number of training examples, it is better to parametrize both the number of iterations $k$ and smoothing parameter $\gamma$ in $n$, and bound the $\Ex_{\phi}(\fh)$ only in terms of $n$. This is given in the following corollary.
\begin{cor} \label{cor:1}
Assume $\gamma \geq 1$ and $B \geq 1$. Paramertize $k$ and $\gamma$ in terms of $n$ as $k+2 = n^{\alpha/2}$ and $\gamma = n^{\beta}$. Then, with a probability at least $1 - \delta$,
\begin{eqnarray}
\Ex_{\phi}(\fh) \leq C(B, t)\left(n^{\beta - \alpha} + n^{\beta - 1} + n^{\beta - (\alpha + 1)/2} + [R_{\phi}^*]^{1/2}n^{(\beta - 1)/2} \right) \label{eqn:bound-in-n}
\end{eqnarray}
where $C(B, t)$ is a constant depending on both $B$ and $t$ with $t = \log(1/\delta) + \log^3n$.
\end{cor}
The bound given in (\ref{eqn:bound-in-n}) depends on $R_{\phi}^*$. We would like to further characterize $R_{\phi}^*$ in terms of $\gamma$. First, we have
\begin{eqnarray*}
\phi(z;\gamma) & = & \max\limits_{\alpha \in [0, 1]} \max(0, 1 - z) + \frac{1}{\gamma}\mathscr{R}(\alpha) \\
& \leq & \max\limits_{\alpha \in [0, 1]} \max(0, 1 - z) + \frac{1}{\gamma}\log 2 =  \phi_{\text{hinge}}(z) + \frac{\log 2}{\gamma},
\end{eqnarray*}
where $\phi_{\text{hinge}}(z) = \max(0, 1 - z)$ is the hinge loss. As a result, we have
\[
R_{\phi}^* \leq R_{\text{hinge}}^* + \frac{\log 2}{\gamma}
\]
where  $R_{\text{hinge}}^* = \min\limits_{\|f\|_{\H} \leq B} \E_{(\x,y)\sim\Pxy}\left[\phi_{\text{hinge}}(yf(\x)) \right]$
is the optimal risk with respect to the hinge loss.  In general, we will assume
\begin{eqnarray}
R_{\phi}^* \leq R_{\text{hinge}}^* + \frac{a}{\gamma^{1 + \xi}} \label{eqn:delta}
\end{eqnarray}
where $a > 0$ is a constant and $\xi \geq 0$ characterizes how  fast $R_{\phi}^*$ will converge to $R_{\text{hinge}}^*$ with increasing $\gamma$. To see why the assumption in (\ref{eqn:delta}) is sensible, consider the case when the optimal classifier $f_{\text{hinge}}^*  = \arg\min_{\|f\|_{\H} \leq B} R_{\text{hinge}}(f)$ can perfectly classify all the data points with margin $\epsilon$, in which we have
\[
R_{\phi}^* \leq R_{\text{hinge}}^* + O\left(\frac{e^{-\epsilon\gamma}}{\gamma}\right)
\]
which satisfy the condition in (\ref{eqn:delta}) with arbitrarily large $\xi$. It is easy to verify that the condition (\ref{eqn:delta}) holds with $\xi > 0$ if $f_{\text{hinge}}^*$ can perfectly classify $O(1 - \gamma^{-1-\xi})$ percentage of data with margin $\epsilon$.

Using the assumption in (\ref{eqn:delta}), we have the following result that characterizes the smooth convex excess risk bound $\Ex_{\phi}(\fh)$ stated in terms of the  parameters $\alpha$, $\delta$ and $R_{\text{hinge}}^*$.
\begin{thm} \label{thm:bound-convex-excess-risk}
Assume $\alpha \geq 1/2$. Set $\beta$ as
\[
\beta = \frac{\min(1/2, \alpha - 1/2)}{1 + \xi}.
\]
With a probability $1 - \delta$, we have
\[
\Ex_{\phi}(\fh) \leq O(n^{-\tau_1} + [R_{\rm{hinge}}^*]^{1/2}n^{-\tau_2})
\]
where
\[
\tau_1 = \frac{1 + 2\xi\min(1, \alpha)}{2(1 + \xi)}, \quad \tau_2 = \frac{1/2 + \xi}{2( 1 + \xi)}
\]
\end{thm}
\begin{proof}
Replacing $R_{\phi}^*$ in Corollary~\ref{cor:1} with the expression in (\ref{eqn:delta}), we have, with a probability $1 - \delta$,
\[
\Ex_{\phi}(\fh) \leq C(R, t, a)\left(n^{\beta - \alpha} + n^{\beta - 1} + n^{\beta - (\alpha + 1)/2} + [R_{\text{hinge}}^*]^{1/2}n^{(\beta - 1)/2} + n^{-1/2-\xi\beta}\right)
\]
We first consider the case when $\alpha > 1$. In this case, we have
\[
\Ex_{\phi}(\fh) \leq O\left(n^{\beta - 1} + n^{-1/2 - \xi\beta} + [R_{\text{hinge}}^*]^{1/2}n^{(\beta - 1)/2}\right)
\]
By choosing $\beta - 1 = -1/2 - \xi \beta$, we have $\beta = \frac{1/2}{1 + \xi}$ and
\[
\Ex_{\phi}(\fh) \leq O(n^{-(1/2+\xi)/(1+\xi)} + [R_{\text{hinge}}^*]^{1/2}n^{-(1/2+\xi)/[2(1+\xi)]}
\]
In the second case, we have $\alpha \in [1/2, 1]$. Hence we  have
\[
\Ex_{\phi}(\fh) \leq O\left(n^{\beta - \alpha} + [R_{\text{hinge}}^*]^{1/2}n^{(\beta - 1)/2} + n^{-1/2-\xi\beta}\right)
\]
By setting $\beta - \alpha = -1/2 - \xi\beta$, we have $\beta = \frac{\alpha - 1/2}{1 + \xi}$ and
\[
\Ex_{\phi}(\fh) \leq O\left(n^{-\frac{\xi\alpha + 1/2}{1 + \xi}} + [R_{\text{hinge}}^*]^{1/2}n^{-(1/2 + \xi)/[2(1+\xi)]}\right).
\]
We complete the proof by combining the results for the two cases.
\end{proof}
\subsection{Bounding binary excess risk $\Ex(f)$}
We now combine the results from Theorem~\ref{thm:bound-binary-excess-risk} and Corollary~\ref{cor:1} to bound $\Ex(f)$.
\begin{thm} \label{thm:combined}
Assume $\alpha \geq 1/2$. For a failure probability $\delta \in (0,1)$, define $n_0$ as
\[
n_0 \leq K_3(B, \delta)\left(\frac{1}{R_{\rm{hinge}}^*}\right)^{1/(2\tau_1 - 2\tau_2)}
\]
where $K_3(B, \delta)$ is a constant depending on $B$ and $\delta$, and $\tau_1$ and $\tau_2$ are defined in Theorem~\ref{thm:bound-convex-excess-risk}. Set $\beta$ as that in Theorem~\ref{thm:bound-convex-excess-risk} if $n \leq n_0$ and $0$, otherwise. Then, with a probability $1 - \delta$, we have
\[
\Ex_{\phi} \leq \left\{
\begin{array}{lc}
K_4(B, \delta) n^{-\tau_1}\log n & n \leq n_0 \\
K_5(B, \delta) n^{-1/2}\log n     & n > n_0
\end{array}
\right.
\]
where $K_4(B, \delta)$ and $K_5(B, \delta)$ are constants depending on $B$ and $\delta$. 
\end{thm}
Theorem~\ref{thm:combined} follows from Theorem~\ref{thm:bound-binary-excess-risk} and similar analysis for Theorem~\ref{thm:bound-convex-excess-risk}, from which we have
\begin{eqnarray*}
\Ex(\fh) = R(\fh) - R_* \leq  O\left(\min\left(\gamma^{-1}, R_{\phi}(\fh) - R_{\phi}^*\right)\log n\right)
\end{eqnarray*}

\begin{remark}
According to Theorem~\ref{thm:combined}, when the number of training examples $n$ is not too large, for the binary excess risk of empirical minimizer we have, with a high probability,
\[
\Ex(\fh) \leq O(n^{-\tau_1}\log n).
\]
In the case when $\xi > 0$ and $\alpha > 1/2$ (i.e. when the number of optimization iterations is larger than $\sqrt{n}$ and $R_{\phi}^*$ converges to $R_{\rm{hinge}}^*$ faster than $1/\gamma$), we have $\tau_1 > 1/2$, implying that using a smooth convex loss will lead to a generalization error bound better than $O(n^{-1/2})$ when the number of training examples is limited.  This implies that  for smooth loss function to  achieve a binary excess error to the extent which is achievable by corresponding non-smooth loss we can run the first order optimization method for a less number of  iterations. This is because our result examines the binary excess risk by taking into account the optimization complexity. 

We also note $1/(2\tau_1 - 2\tau_2)$ is given by
\[
\frac{1}{2\tau_1 - 2\tau_2} = \frac{1 + \xi}{1/2 + \xi\min(1, 2\alpha - 1)}
\]
When $\alpha \leq 3/4$, we have $n_0 \geq K_3(B, \delta) [R^*_{\rm{hinge}}]^{-2}$, which could be a large number when $R_{\rm{hinge}}^*$ is very small.
\end{remark}
%
\section{Conclusion}\label{sec:conclusion}
In this paper we have investigated how the smoothness of loss function being used as the surrogate of 0-1 loss function in empirical risk minimization affects the excess binary risk.  While the relation between convex excess risk and binary excess risk being provably established previously under weakest possible condition such as differentiability, it was not immediately obvious how smoothness of convex surrogate  relates to  statistical consequences. This paper made first step towards understanding this affect.   In particular, in contrast  to optimization and generalization analysis that favor smooth surrogate losses, our results revealed that smoothness  degrades the binary excess risk. To investigate guarantees on which the smoothness would be a desirable property, we proposed a unified analysis that integrates errors in optimization, generalization, and translating convex excess risk into binary excess risk. Our result shows that under favorable conditions and with appropriate choice of smoothness parameter, a smoothed hinge loss can achieve a binary excess risk that is better than $O(1/\sqrt{n})$.  

%
\appendix
%
\section*{Appendix A. Proof of Theorem~\ref{thm:psi}}\label{app:a}
We first compute
\[
z = \mathop{\arg\min}\limits_{z'} \frac{1 + \eta}{2}\phi(z'; \gamma) + \frac{1 - \eta}{2}\phi(-z'; \gamma)
\]
By setting the derivative to be zero, we have
\begin{eqnarray*}
\frac{1+\eta}{1 + \exp(-\gamma (1 - z))} = \frac{1 - \eta}{1 + \exp(-\gamma(1 + z))}
\end{eqnarray*}
and therefore
\[
(1 + \eta)\exp(-\gamma z) - (1 - \eta)\exp(\gamma z) + 2\eta\exp(\gamma) = 0.
\]
Solving the equation, we obtain
\[
\exp(-\gamma z) = \frac{-\eta \exp(\gamma) + \sqrt{\eta^2\exp(2\gamma) + (1- \eta^2)}}{1 + \eta}
\]
and
\[
\exp(\gamma z) = \frac{\eta \exp(\gamma) + \sqrt{\eta^2\exp(2\gamma) + (1- \eta^2)}}{1 - \eta}.
\]
It is easy to verify that $\sgn(z) = \sgn(\eta)$. This is because if $\eta > 0$, we have
\[
\exp(-\gamma z) \leq \frac{\sqrt{1 - \eta^2}}{1 + \eta} = \sqrt{\frac{1 - \eta}{1 + \eta}} < 1
\]
and therefore $z > 0$. On the other hand, when $\eta < 0$, we have
\[
\exp(\gamma z) = \frac{1 + \eta}{-\eta \exp(\gamma) + \sqrt{\eta^2\exp(2\gamma) + (1- \eta^2)}} \leq \sqrt{\frac{1 +\eta}{1 - \eta}} < 1,
\]
and therefore $z < 0$. Using the solution for $z$, we compute $\phi(\eta)$ as
\begin{eqnarray*}
\psi(\eta;\gamma) & = & \frac{1 + \eta}{2}\phi(z;\gamma) + \frac{1 - \eta}{2}\phi(z;\gamma) - \min\limits_z \frac{1 + \eta}{2}\phi(z;\gamma) + \frac{1 - \eta}{2}\phi(z;\gamma) \\
& = & -\frac{1 +\eta}{2\gamma} \log \frac{1 + \exp(\gamma(1 - z))}{1 + \exp(\gamma)} - \frac{1 - \eta}{2\gamma} \log\frac{1 + \exp(\gamma(1 + z))}{1 + \exp(\gamma)}.
\end{eqnarray*}
By defining  constants $C_1= -\eta e^{\gamma} + \sqrt{\eta^2e^{2\gamma} + 1 - \eta^2}$ and  $C_2 = \eta e^{\gamma} + \sqrt{\eta^2e^{2\gamma} + 1 - \eta^2}$, we can rewrite the transform function $\psi(\eta;\gamma)$ as
\[
\psi(\eta; \gamma) = -\frac{1 + \eta}{2\gamma}\log\left(\frac{1}{1 + e^{\gamma}}\left[1 + e^{\gamma} \frac{C_1}{1 + \eta}\right]\right) -\frac{1 - \eta}{2\gamma}\log\left(\frac{1}{1 + e^{\gamma}}\left[1 + e^{\gamma} \frac{C_2}{1 - \eta}\right]\right).
\]
%
\section*{Appendix B. Proof of Theorem~\ref{thm:psi-simple}}
Since the expression for $\psi(\eta;\gamma)$ is symmetric in terms $\eta$, we will only consider the case when $\eta > 0$. First, we have
\[
\frac{C_1e^{\gamma}}{1 + \eta} = \frac{1 - \eta}{\eta + \sqrt{\eta^2 + (1 - \eta^2)e^{-2\gamma}}} \leq \frac{1 - \eta}{2\eta}.
\]
Similarly, we have
\[
\frac{C_2 e^{\gamma}}{1 - \eta} = \frac{e^{\gamma}}{1 - \eta}\left(\eta e^{\gamma} + \sqrt{\eta^2e^{2\gamma} + 1 - \eta^2}\right) \leq \frac{1+\eta}{1 - \eta} e^{2\gamma}
\]
Thus, we have
\begin{eqnarray*}
\psi(\eta; \gamma) & \geq & \frac{1+\eta}{2\gamma}\log(1 + e^{\gamma}) - \frac{1 + \eta}{2\gamma}\log\left( \frac{1-\eta}{2\eta}\right) - \frac{1 - \eta}{2\gamma}\log\left(\frac{1+\eta}{1 - \eta} e^{\gamma}\right) \\
& \geq & \eta - \frac{1 + \eta}{2\gamma}\log\left( \frac{1-\eta}{2\eta}\right)-\frac{1 - \eta}{2\gamma}\log\left(\frac{1+\eta}{1 - \eta} \right) \\
& \geq & \eta - \frac{1}{\gamma}\log\left(\frac{1 - \eta^2}{4\eta} + \frac{1 + \eta}{2} \right)  =  \eta - \frac{1}{\gamma}\log\left(\frac{1}{4\eta} + \frac{\eta}{4} + \frac{1}{2}\right)
\end{eqnarray*}
where the last inequality follows from the concaveness of $\log(\cdot)$ function. As a result when $\eta \in (-1, 1)$ we have
\[
\frac{1}{4\eta} + \frac{\eta}{4} + \frac{1}{2} \leq \frac{1}{\eta},
\]
which completes the proof.

%
\section*{Appendix C. Proof of Theorem~\ref{thm:combined-1}}
Applying  Lemmas~\ref{lemma:opt} and~\ref{lemma:generalization} to the solution to  the  empirical convex risk minimizer $\fh$, we have
\begin{eqnarray}
\lefteqn{R_{\phi}(\fh) \leq \Rh_{\phi}(\fh) + K_1\left(\frac{(B + \gamma B^2)t}{n} + \sqrt{\Rh_{\phi}(\fh)\frac{(B + \gamma B^2)t}{n}}\right)}  \label{eqn:bound-1} \\
& \leq & \Rh_{\phi}(f_{\gamma}^*) + \frac{\gamma B^2}{(k+2)^2} + K_1\left(\frac{(B + \gamma B^2)t}{n} + \sqrt{\Rh_{\phi}(f_{\gamma}^*)\frac{(B + \gamma B^2)t}{n}} + \sqrt{\frac{\gamma B^2(B + \gamma B^2)t}{(k+2)^2n}}\right) \nonumber
\end{eqnarray}
On the other hand, by the application of the Bernstein's inequality~\cite{boucheron2004concentration}, with probability at least $1-\delta$ we have
\begin{eqnarray}
\begin{aligned}
\Rh_{\phi}(f_{\gamma}^*) - R_{\phi}(f_{\gamma}^*) &\leq \frac{4B\log \frac{1}{\delta}}{n} + \sqrt{ \frac{4 \E_{(\x,y)\sim\Pxy} \left[ \left(\phi(yf_{\gamma}^*(\x);\gamma) - R_{\phi}(f_{\gamma}^*)\right)^2\right]\log\frac{1}{\delta}}{n}}{} \\
&\leq \frac{4B \log \frac{1}{\delta}}{n} + \sqrt{\frac{8BR_{\phi}(f_{\gamma}^*) \log\frac{1}{\delta}}{n}}. \label{eqn:bound-2}
\end{aligned}
\end{eqnarray}
We conclude the proof by plugging in (\ref{eqn:bound-1}) with (\ref{eqn:bound-2}), replacing the constants  with a new universal constant $K$, and noting that $t = \log \frac{1}{\delta} + \log^3n$ .

%
\bibliographystyle{plain}
\bibliography{smooth_binary_risk}

\end{document}